\documentclass{article}
\usepackage{spconf,amsmath}
\usepackage{cite}
\usepackage{algorithm}
\usepackage{amssymb}
\usepackage{epsfig}
\usepackage{array}
\usepackage{subfigure}
\usepackage{bm}
\usepackage{color}
\usepackage{latexsym}
\usepackage{multirow}
\usepackage{algorithm}
\usepackage{algorithmic}
\usepackage{graphicx}
\usepackage{epstopdf}
\usepackage{amsthm}

\newtheorem{theorem}{Theorem}
\newtheorem{lemma}[theorem]{Lemma}

\title{Graph Fourier Transform with Negative Edges for Depth Image Coding}
\name{Weng-Tai Su{\small $~^{\ast}$},
Gene Cheung{\small $~^{\#}$},
Chia-Wen Lin{\small $~^{\ast}$}}
\address{$^{\ast}$\,National Tsing Hua University,
$^{\#}$\,National Institute of Informatics}
\begin{document}
%
\maketitle
\begin{abstract}
Recent advent in graph signal processing (GSP) has led to the development of new graph-based transforms and wavelets for image / video coding, where the underlying graph describes inter-pixel correlations.
In this paper, we develop a new transform called signed graph Fourier transform (SGFT), where the underlying graph $\mathcal{G}$ contains negative edges that describe anti-correlations between pixel pairs.
Specifically, we first construct a one-state Markov process that models both inter-pixel correlations and anti-correlations.
We then derive the corresponding precision matrix, and show that the loopy graph Laplacian matrix $\mathbf{Q}$ of a graph $\mathcal{G}$ with a negative edge and two self-loops at its end nodes is approximately equivalent. This proves that the eigenvectors of $\mathbf{Q}$---called SGFT---approximates the optimal Karhunen-Lo\`eve Transform (KLT).
We show the importance of the self-loops in $\mathcal{G}$ to ensure $\mathbf{Q}$ is positive semi-definite.
We prove that the first eigenvector of $\mathbf{Q}$ is piecewise constant (PWC), and thus can well approximate a piecewise smooth (PWS) signal like a depth image.
Experimental results show that a block-based coding scheme based on SGFT outperforms a previous scheme using graph transforms with only positive edges for several depth images.
\end{abstract}

\vspace{-0.08in}
\section{Introduction}
\label{sec:intro}
\vspace{-0.05in}
The advent of \textit{graph signal processing} (GSP) \cite{shuman13}---the study of signals that live on irregular data kernels described by graphs---has led to the development of new graph-based tools for coding of images and videos \cite{shen10pcs,hu12icip,hu15,pavez15,hu15spl,rotondo15,narang09,chao15}.
Among them are variants of \textit{graph Fourier transforms} (GFT) \cite{shen10pcs,hu12icip,hu15,pavez15,hu15spl,rotondo15} for compact signal representation in the transform domain, where an underlying graph reflects inter-pixel correlations.
Because a graphical model is versatile in describing correlation patterns in a pixel patch, recent works like \cite{hu15} have shown significant coding gain over state-of-the-art codecs like HEVC for piecewise smooth (PWS) images like depth maps.

Opposite to the notion of ``correlation" or ``similarity" is the notion of ``anti-correlation" or ``dissimilarity".
If two variables $i$ and $j$ are anti-correlated, then their respective sample values $x_i$ and $x_j$ are very different with a high probability.
We model anti-correlation with a \textit{negative} edge with weight $w_{i,j} < 0$ connecting nodes $i$ and $j$.
The meaning of a negative edge is very different from no edge, which implies conditional independence between the two variables for a Gaussian Markov Random Field (GMRF) model. 
Recent research in data mining \cite{kunegis10}, control \cite{zelazo14,chen16} and social network analysis \cite{chu16} has shown that explicitly expressing anti-correlation in a graphical model can lead to enhanced performance in different problem domains. 

Inspired by these earlier works \cite{kunegis10,zelazo14,chen16,chu16}, in this paper we develop a new transform called \textit{signed graph Fourier transform} (SGFT), where the underlying graph $\mathcal{G}$ contains negative edges that describe anti-correlations between pixel pairs.
Specifically, we first construct a one-state Markov process that models both inter-pixel correlations and anti-correlations in an $N$-pixel row, and derive the corresponding precision matrix $\mathbf{P}$. 
We then design an $N$-node graph $\mathcal{G}$ with a negative edge and two self-loops at its end nodes, and show that the corresponding loopy graph Laplacian matrix $\mathbf{Q}$ \cite{dorfler13}---the sum of the graph Laplacian matrix and a diagonal matrix containing self-loop weights---is approximately equivalent to $\mathbf{P}$. 
This proves that the eigenvectors of $\mathbf{Q}$---called SGFT---approximates the optimal \textit{Karhunen-Lo\`eve Transform} (KLT) in signal decorrelation.

Moreover, we show the importance of the self-loops in $\mathcal{G}$ to guarantee that $\mathbf{Q}$ is positive semi-definite, and hence its eigenvalues are non-negative and can be properly interpreted as graph frequencies.
We prove that the first eigenvector of $\mathbf{Q}$ is piecewise constant (PWC), and thus can well approximate a PWS signal like a depth image.
Experimental results show that a block-based coding scheme based on SGFT outperforms a previous proposal \cite{hu15} using graph transforms with only positive edges for several depth images.

The outline of the paper is as follows.
In Section \ref{sec:transform}, we describe a one-state Markov process, and show that the loopy graph Laplacian $\mathbf{Q}$ of a carefully constructed graph is equivalent to the corresponding precision matrix.
We describe our depth map coding algorithm based on SGFT in Section \ref{sec:code}.
Experimental results and conclusion are presented in Section \ref{sec:results} and \ref{sec:conclude}, respectively.


\vspace{-0.08in}
\section{Signed Graph Fourier Transform}
\label{sec:transform}
\vspace{-0.05in}
\subsection{Markov Process with Anti-Correlation}
\label{subsec:markov}

\vspace{-0.05in}
As done in previous signal decorrelation analysis \cite{han12,hu15,hu15spl}, we assume a one-state Markov process of length $N$ for 1D variable vector $\mathbf{x}$.
Specifically, we assume first that the first pixel $x_1$ is a zero-mean random variable $z_1$ with variance $\sigma_1^2$.
We then assume that the difference between a new pixel $x_i$ and a previous pixel $x_{i-1}$ is a zero-mean random variable $z_i$ with variance $\sigma_i^2$.

The exception is the $k$-th variable $x_k$, where we assume that the \textit{sum} of $x_k$ and $x_{k-1}$ is a zero-mean random variable $z_k$ with variance $\sigma_k^2$.
Assuming that $x_i \in [-R, R]$, this assumption means $x_k$ and $x_{k-1}$ are \textit{anti-correlated}; \textit{i.e.}, if $x_{k-1}$ is a large positive (negative) number, then $x_k$ is a large negative (positive) number with high probability.
We summarize the equations below:

\vspace{-0.3in}
\begin{align}
x_1 & = z_1 \nonumber \\
x_2 - x_1 & = z_2 \nonumber \\
\vdots \nonumber \\
x_k + x_{k-1} & = z_k \nonumber \\
\vdots \nonumber \\
x_N - x_{N-1} & = z_N
\end{align}
We can write the above in matrix form:
\begin{align}
\underbrace{\left[ \begin{array}{ccccc}
1 & 0 & 0 & \ldots & 0 \\
-1 & 1 & 0 & \ldots & 0 \\
\vdots & \ddots & & & \\
\ldots 0 & 1 & 1 & 0 \ldots \\
\vdots & & & \ddots & \\
0 & \ldots & 0 & -1 & 1
\end{array} \right]}_{\mathbf{M}}
\mathbf{x} = \mathbf{z}
\end{align}
or $\mathbf{x} = \mathbf{M}^{-1} \mathbf{z}$.
We see that the mean $\bar{\mathbf{x}}$ of variable $\mathbf{x}$ is $E[\mathbf{x}] = \mathbf{M}^{-1} E[\mathbf{z}] = \mathbf{0}$.

We now derive the covariance matrix $\mathbf{C}$ of $\mathbf{x}$:
\begin{align}
\mathbf{C} & = E[(\mathbf{x}-\bar{\mathbf{x}})(\mathbf{x}-\bar{\mathbf{x}})^{\top}]
= E[\mathbf{x} \mathbf{x}^{\top}] \nonumber \\
& = \mathbf{M}^{-1} \underbrace{E[ \mathbf{z} \mathbf{z}^{\top} ]}_{\mathrm{diag}(\{\sigma_i^2)\}} (\mathbf{M}^{-1})^{\top}
\label{eq:covariance}
\end{align}
The precision matrix $\mathbf{P}$ is the inverse of $\mathbf{C}$ and shares the same eigenvectors:
\begin{align}
\mathbf{P} & = \mathbf{C}^{-1} \nonumber \\
& = \mathbf{M}^{\top} \mathrm{diag}(\{1/\sigma_i^2\}) \mathbf{M}
\label{eq:precision}
\end{align}
which can be expanded to:
\begin{scriptsize}
\begin{align}
= \left[ \begin{array}{ccccc}
\frac{1}{\sigma_1^2} + \frac{1}{\sigma_2^2} & - \frac{1}{\sigma_2^2} & 0 & \ldots & \\
-\frac{1}{\sigma_2^2} & \frac{1}{\sigma_2^2} + \frac{1}{\sigma_3^2} & - \frac{1}{\sigma_3^2} & 0 & \ldots \\
\vdots & \ddots & & & \\
0 \ldots 0 & -\frac{1}{\sigma_{k-1}^2} & \frac{1}{\sigma_{k-1}^2} + \frac{1}{\sigma_{k}^2} &
\frac{1}{\sigma_k^2} & 0 \ldots \\
0 \ldots 0 & 0 & \frac{1}{\sigma_k^2} & \frac{1}{\sigma_k^2} + \frac{1}{\sigma_{k+1}^2} & -\frac{1}{\sigma_{k+1}^2} ~ 0 \ldots \\
\vdots & & & \ddots & \\
0 & \ldots & 0 & -\frac{1}{\sigma_{N}^2} & \frac{1}{\sigma_N^2}
\end{array}
\right] \nonumber
\end{align}
\end{scriptsize}
Note that $\mathbf{C}$ is always invertible since $\sigma_i^2 > 0, \forall i$.

\vspace{-0.05in}
\subsection{Optimal Graph Construction}

\vspace{-0.2in}
\begin{figure}[htb]
  \centering
  \centerline{\includegraphics[width=9.0cm]{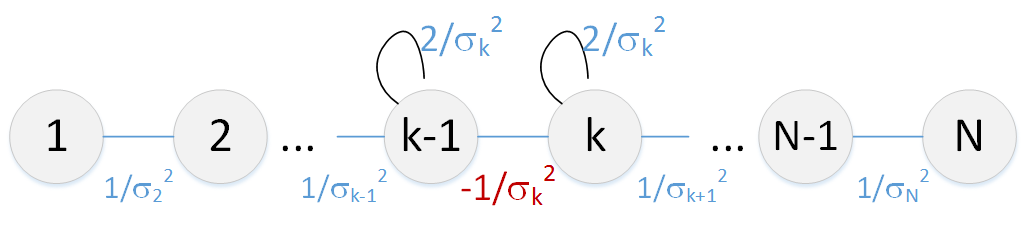}}
\vspace{-0.15in}
\caption{\small{Line graph construction with one negative edge at node pair $(k-1,k)$ and two self-loops at nodes $k-1$ and $k$.}}
\label{fig:lineGraph}
\end{figure}

\vspace{-0.05in}
\subsubsection{Loopy Graph Laplacian}

\vspace{-0.05in}
We define a graph $\mathcal{G}(\mathcal{V}, \mathcal{E})$ with positive / negative edges and self-loops as follows.
There are $N$ nodes in node set $\mathcal{V}$.
Each node $i$ is connected to a neighboring node $j$ with an edge $\mathcal{E}$ if the $(i,j)$-th entry in the \textit{adjacency matrix} $\mathbf{A} \in \mathbb{R}^{N \times N}$ is non-zero, \textit{i.e.}, edge weight $A_{i,j} \neq 0$.
Because edges are undirected, $\mathbf{A}$ is symmetric.
We assume that $\mathcal{G}$ contains self-loops (positive edges to oneself), which means $A_{i,i} > 0$ for some $i$.
We define a diagonal \textit{degree matrix} $\mathbf{D} \in \mathbb{R}^{N \times N}$ as a function of $\mathbf{A}$: $D_{i,i} = \sum_{j} A_{i,j}$.
Given $\mathbf{A}$ and $\mathbf{D}$, we define the \textit{graph Laplacian matrix} $\mathbf{L} = \mathbf{D} - \mathbf{A}$, as conventionally done in the GSP literature\cite{shuman13}.

Graph Laplacian $\mathbf{L}$ does not reflect weights of the self-loops; $\mathbf{D}$ cancels out the diagonal entries in $\mathbf{A}$.
Following \cite{dorfler13}, we define a \textit{loopy graph Laplacian matrix} $\mathbf{Q} = \mathbf{L} + \textrm{diag}(\{A_{i,i}\})$ that includes contributions from self-loops.
A loopy Laplacian is an example of  a \textit{generalized graph Laplacian} \cite{biyikoglu05}, which is generally defined as the sum of a graph Laplacian matrix $\mathbf{L}$ and a diagonal matrix.

Loopy Laplacian $\mathbf{Q}$ is a symmetric, real matrix, and thus admits a set of orthogonal eigenvectors $\boldsymbol{\phi}_i$ with real eigenvalues $\lambda_i$.
Similarly done in the GSP literature \cite{shuman13}, we define here the \textit{signed graph Fourier transform} (SGFT) as the set of eigenvectors $\boldsymbol{\Phi}$ for the loopy Laplacian $\mathbf{Q}$ for a graph with negative edges.

\vspace{-0.05in}
\subsubsection{Optimal Decorrelation Transform}

\vspace{-0.05in}
We now construct a graph with self-loops, so that the resulting loopy Laplacian approximates the precision matrix $\mathbf{P}$ defined in Section \ref{subsec:markov}.
We construct an $N$-node line graph, where the $(i,i-1)$-th edge weight is assigned as follows:

\vspace{-0.1in}
\begin{small}
\begin{align}
A_{i,i-1} = \left\{ \begin{array}{ll}
1/\sigma_i^2 & \mbox{if} ~ i \in \{1, \ldots, k-1\} \cup \{k+1, \ldots, N\} \\
- 1/\sigma_i^2 & \mbox{if} ~ i = k
\end{array} \right.
\label{eq:weight1}
\end{align}
\end{small}
In other words, there is a positive edge between every node pair $(i,i-1)$ with weight $1/\sigma_i^2$, except bewteen pair $(k,k-1)$, where there is a negative edge with weight $-1/\sigma_k^2$.

Next, we add self-loops to the two nodes $k-1$ and $k$ connected by the lone negative edge:
\begin{align}
A_{i,i} = \left\{ \begin{array}{ll}
2/\sigma_k^2 & \mbox{if} ~ i \in \{k-1, k\} \\
0 & \mbox{o.w.}
\end{array} \right.
\label{eq:weight2}
\end{align}
One can now verify that the loopy graph Laplacian $\mathbf{Q}$ for this constructed graph $\mathcal{G}$ is the precision matrix $\mathbf{P}$ as $\sigma_1^2 \rightarrow \infty$.
Variance $\sigma_1^2$ of the first pixel $x_1$ tends to be large, so in practice $\mathbf{Q} \approx \mathbf{P}$.

We know that the eigenvectors of the precision matrix $\mathbf{P}$ compose the basis vectors of the \textit{Karhunen-Lo\`eve Transform} (KLT), which optimally decorrelates an input signal following a statistical model.
Because our loopy Laplacian $\mathbf{Q} \approx \mathbf{P}$, the SGFT $\boldsymbol{\Phi}$ of $\mathbf{Q}$ also approximates the KLT.
We can thus claim the following:

\vspace{-0.05in}
\begin{small}
\begin{quote}
Constructed graph $\mathcal{G}$ with one negative edge and two self-loops, where edge weights are assigned according to (\ref{eq:weight1}) and (\ref{eq:weight2}), is the \textit{optimal graph}, whose corresponding SGFT optimally decorrelates the input signal.
\end{quote}
\end{small}

\vspace{-0.05in}
\subsubsection{Definiteness of Loopy Graph Laplacian}

\vspace{-0.05in}
By definition in (\ref{eq:precision}), we see that the precision matrix $\mathbf{P}$ is \textit{positive semi-definite} (PSD):
\begin{align}
\mathbf{x}^{\top} \mathbf{P} \mathbf{x} & =
\mathbf{x}^{\top} \mathbf{M}^{\top} \textrm{diag}(\{1/\sigma_i^2\}) \mathbf{M} \, \mathbf{x} \nonumber \\
& = \| \textrm{diag}(\{1/\sigma_i\}) \mathbf{M} \, \mathbf{x} \|_2^2 \geq 0
\end{align}
The positive semi-definiteness of $\mathbf{P}$---and hence loopy Laplacian $\mathbf{Q}$ as $\sigma_1^2 \rightarrow \infty$---is ensured thanks to the self-loops introduced at the two end nodes of the negative edge.

To see the importance of the two self-loops with proper weights, consider the loopy graph Laplacian $\mathbf{Q}$ with self-loop weight $2/\sigma_k^2 - \epsilon$, $\epsilon > 0$.
The $(k-1)$-th and $k$-th entries of rows $(k-1)$ and $k$ of $\mathbf{Q}$ are then:

\vspace{-0.1in}
\begin{small}
\begin{align}
\left[ \begin{array}{cc}
\left( \frac{1}{\sigma_{k-1}^2} + \left( \frac{1}{\sigma_{k}^2} - \epsilon \right) \right) & \frac{1}{\sigma_k^2} \\
\frac{1}{\sigma_k^2} & \left(\left(\frac{1}{\sigma_k^2} - \epsilon \right) + \frac{1}{\sigma_{k+1}^2} \right)
\end{array}
\right] \nonumber
\end{align}
\end{small}\noindent
where $\epsilon = 0$ would imply that each self-loop has weight exactly $2/\sigma_k^2$.
We show that there exists edge weights $1/\sigma_{k-1}^2$, $-1/\sigma_k^2$ and $1/\sigma_{k+1}^2$ so that $\mathbf{Q}$ is indefinite.

We first define the \textit{inertia} $\mathrm{In}(\mathbf{Q})$ of $\mathbf{Q}$, where $\mathrm{In}(\mathbf{Q}) = (i^+(\mathbf{Q}), i^-(\mathbf{Q}), i^0(\mathbf{Q}))$ is a triple counting the positive, negative and zero eigenvalues of $\mathbf{Q}$.
Suppose we divide nodes in $\mathbf{Q}$ into two sets and partition $\mathbf{Q}$ accordingly:
\begin{equation}
\mathbf{Q} = \left[ \begin{array}{cc}
\mathbf{Q}_{1,1} & \mathbf{Q}_{1,2} \\
\mathbf{Q}_{1,2}^{\top} & \mathbf{Q}_{2,2}
\end{array}
\right]
\end{equation}
According to the \textit{Haysworth Inertia additivity} formula \cite{haynsworth68}, $\mathrm{In}(\mathbf{Q})$ can be computed in parts:
\begin{equation}
\mathrm{In}(\mathbf{Q}) = \mathrm{In}(\mathbf{Q}_{1,1}) +
\mathrm{In}(\mathbf{Q} / \mathbf{Q}_{1,1})
\label{eq:haysworth}
\end{equation}
where $\mathbf{Q} / \mathbf{Q}_{1,1}$ is the \textit{Schur Complement}\footnote{https://en.wikipedia.org/wiki/Schur\_complement} (SC) of block $\mathbf{Q}_{1,1}$ of matrix $\mathbf{Q}$.
Suppose we choose set $1$ to be nodes $k-1$ and $k$.
The determinant of $\mathbf{Q}_{1,1}$ can be written as:

\vspace{-0.05in}
\begin{small}
\begin{align}
|\mathbf{Q}_{1,1}| & =
\frac{1}{\sigma_{k-1}^2}\left(\frac{1}{\sigma_k^2} - \epsilon \right) +
\frac{1}{\sigma_{k-1}^2 \sigma_{k+1}^2} +
\left( \frac{1}{\sigma_k^2} - \epsilon \right)^2 + \nonumber \\
& \left( \frac{1}{\sigma_k^2} - \epsilon \right) \frac{1}{\sigma_{k+1}^2} -
\frac{1}{\sigma_k^4}
\end{align}
\end{small}
Suppose that $\sigma_{k-1}^2, \sigma_{k+1}^2 \gg \sigma_k^2$, then $|\mathbf{Q}_{1,1}|$ simplifies to:
\begin{align}
|\mathbf{Q}_{1,1}| \approx
\left( \frac{1}{\sigma_k^2} - \epsilon \right)^2 - \frac{1}{\sigma_k^4}
\end{align}
which is negative for small $\epsilon > 0$.
This implies that inertia $\mathrm{In}(\mathbf{Q}_{1,1})$ has at least one negative eigenvalue.
From (\ref{eq:haysworth}), it implies also that $\mathbf{Q}$ has at least one negative eigenvalue, and $\mathbf{Q}$ is indefinite.

The important lesson from the above analysis is the following:
\textit{our constructed loopy Laplacian $\mathbf{Q}$ requires properly weighted self-loops to be PSD, so that its eigenvalues can be properly interpreted as graph frequencies and its eigenvectors as graph frequency components.}

\vspace{-0.05in}
\subsubsection{PWS Signal Approximation}

\vspace{-0.05in}
To see more intuitively why basis vectors in SGFT can compactly approximate PWS signals, we show that the first eigenvector $\boldsymbol{\phi}_1$ of the loopy Laplacian $\mathbf{Q}$ corresponding to eigenvalue $\lambda_1 = 0$ is a \textit{piecewise constant} (PWC) signal.
Specifically, we define a PWC vector $\mathbf{v}$ as follow:
\begin{equation}
v_i = \left\{ \begin{array}{ll}
1 & \mbox{if} ~ 1 \leq i < k \\
-1 & \mbox{if} ~ k \leq i \leq N
\end{array} \right.
\end{equation}
We state the following claim formally.
\begin{lemma}
$\mathbf{v}$ is the first (unnormalized) eigenvector $\boldsymbol{\phi}_1$ of loopy Laplacian $\mathbf{Q}$ corresponding to eigenvalue $\lambda_1 = 0$.
\end{lemma}

\begin{proof}
Examining the entries in $\mathbf{Q}$ (precision matrix $\mathbf{P}$ in (\ref{eq:precision}) for $\sigma_1^2 = \infty$), we see that, with the exception of $(k-1)$-th and $k$-th rows, each row $i$ satisfies the condition $Q_{i,i} = - \sum_{j | j \neq i} Q_{i,j}$.
Hence $\mathbf{v}$ with the same constant value for entries $i-1$ to $i+1$ of row $i$ (if they exist) will sum to $0$.
For the $(k-1)$-th and $k$-th rows, if their respective off-diagonal entries $k$ and $k-1$ have negative sign instead, then again for each row the sum of off-diagonal entries equals the diagonal entry.
In $\mathbf{v}$, entries $k-2$ and $k-1$ have the opposite sign (but same magnitude) as entries $k$ and $k+1$, hence multiplying $\mathbf{v}$ to $(k-1)$-th and $k$-th rows will also result in $0$.
\end{proof}

\begin{figure}[htb]
\begin{minipage}[b]{.23\textwidth}
  \centering
  \centerline{\includegraphics[width=4.5cm]{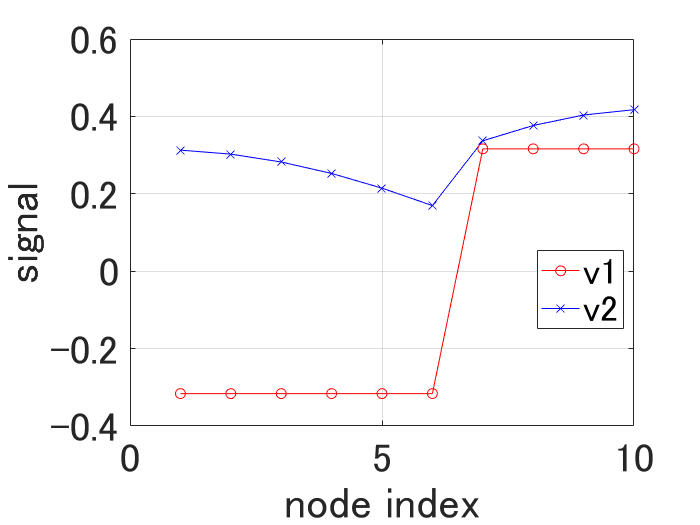}}
  \centerline{\small{(a) SGFT basis vectors}}\medskip
\end{minipage}
\hfill
\begin{minipage}[b]{.23\textwidth}
  \centering
  \centerline{\includegraphics[width=4.5cm]{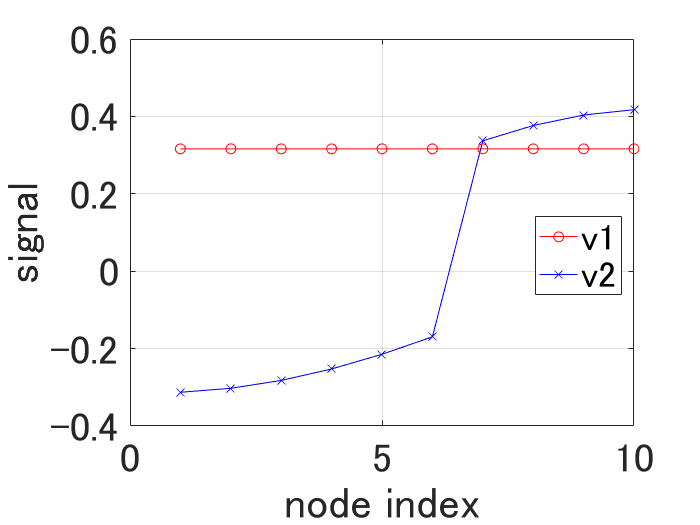}}
  \centerline{\small{(b) GFT basis vectors}}\medskip
\end{minipage}
\vspace{-0.2in}
\caption{\small{First two eigenvectors of: a) loopy Laplacian $\mathbf{Q}$ for a 10-node graph with negative edge weight $-0.1$ between nodes $6$ and $7$; b) graph Laplacian $\mathbf{L}$ for the same graph with small edge weight $0.1$ between nodes $6$ and $7$. Other edge weights are $1$.}}
\label{fig:ex1}
\end{figure}

This means that the first eigenvector $\boldsymbol{\phi}_1$ of $\mathbf{Q}$ alone can well approximate the shape of a PWS signal.
This is in contrast to the second eigenvector of graph Laplacian $\mathbf{L}$ with a small positive edge weight across node pair $(k-1,k)$, which approaches PWC behavior as the small weight tends to $0$.
See Fig.\;\ref{fig:ex1} for an illustration of the first two eigenvectors of $\mathbf{Q}$ for a $10$-node line graph with a negative edge of weight $-0.1$, and the first two eigenvectors of $\mathbf{L}$ for the same graph with the negative edge replaced by a positive edge of weight $0.1$.

\vspace{-0.08in}
\section{Depth Image Coding}
\label{sec:code}
\vspace{-0.05in}
Inspired by the analysis for the 1D case in Section \ref{sec:transform}, we construct a depth image coding scheme where each $N \times N$ block is coded using an appropriate graph.
As done in \cite{hu15}, we assume first that object contours in the image are detected and encoded efficiently using \textit{arithmetic edge coding} (AEC) \cite{daribo14} as side information (SI). 
For a given block, if there are no contours that cross it, then the block is sufficiently smooth and is coded using DCT.
If there is a contour that crosses the block, then we perform SGFT transform coding as follows. 

We first draw a 4-connected graph $\mathcal{G}$ for a $N \times N$ block; \textit{i.e.}, each pixel is represented by a node and is connected to its four horizontal and vertical adjacent pixels. 
For each connected node pair that do not cross a detected contour, we assign a \textit{positive} edge weight $1$. 
For each connected node pair $(i,j)$ that cross a contour, we assign a \textit{negative} weight $-w < 0$, where $w > 0$, and add a self-loop of weight $2 w$ to each end node. 
We tune $w$ per image and the value is encoded separately.
Because the graph construction depends only on the coded contours, there is no additional overhead to code the graph explicitly.
Having constructed graph $\mathcal{G}$, we compute the loopy Laplacian $\mathbf{Q}$ and its eigenvectors $\boldsymbol{\Phi}$ as the SGFT matrix for transform coding. 
SGFT coefficients are quantized and entropy coded as done in \cite{hu15}.

\vspace{-0.08in}
\section{Experiments}
\label{sec:results}
\vspace{-0.05in}
To evaluate the coding performance of our proposed SGFT for PWS depth images, we use two $448 \times 368$ depth images from the Middlebury dataset\footnote{http://vision.middlebury.edu/stereo/}: \texttt{Teddy} and \texttt{Cones}.
We compare for the two images the rate-PSNR performance of our proposed SGFT against DCT and weighted GFT (WGFT) proposed in \cite{hu15}, which uses a pre-trained non-negative weight to represent the weak correlation between two spatially adjacent pixels that cross a detected image contour.
For SGFT, we search for the optimal negative edge weight per image, which is transmitted as SI.
Following the coding scheme proposed in \cite{hu15}, as explained in Section \ref{sec:code}, we only perform SGFT / WGFT on edge blocks which are detected and coded using AEC \cite{daribo14}.
The block size of SGFT and WGFT is set to $4 \times 4$, and that of DCT is $8 \times 8$.
We use the single-resolution implementation of WGFT  in \cite{hu15}.
Edge-aware intra-prediction \cite{shen10icip} is performed per block prior to transform coding of the depth block; thus the prediction residual block is much closer to an AC signal than the original block, and our statistical model discussed in Section \ref{subsec:markov} is a reasonable fit.
The set of quantization parameters (QP) used for SGFT and WGFT is QP = [16 24 32 40 48], whereas QP = [40 42 44 46 48] for DCT.

Fig.\;\ref{fig:RD_Performance} compares the Rate-PSNR performances of SGFT, WGFT, and DCT for \texttt{Teddy} and \texttt{Cones} for a typical PSNR range.
As shown in Fig.\;\ref{fig:RD_Performance}, both SGFT and WGFT significantly outperform DCT by up to $5$dB for \texttt{Teddy} and $6$dB for \texttt{Cones} in PSNR.
Our proposed SGFT achieves further $0.3$ to $0.5$dB coding gain in PSNR compared to WGFT at some bitrates.
Though the additional coding gain from SGFT is not very large, \textit{we have empirically demonstrated, for the first time in the literature, that a statistical model specifying anti-correlation---and its associated optimal decorrelation graph transform in SGFT---can be effectively used in an image coding scenario}.

\begin{figure}

\begin{minipage}[b]{.48\linewidth}
 \centering
 \centerline{\includegraphics[width=4.45cm]{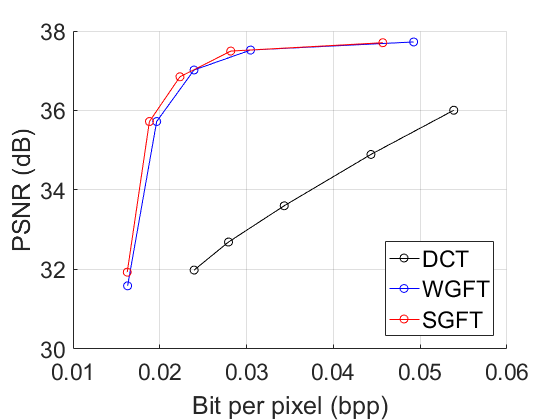}}
 \centerline{(a) \texttt{teddy}}\medskip
\end{minipage}
\vspace{-0.15in}
\begin{minipage}[b]{.48\linewidth}
 \centering
 \centerline{\includegraphics[width=4.45cm]{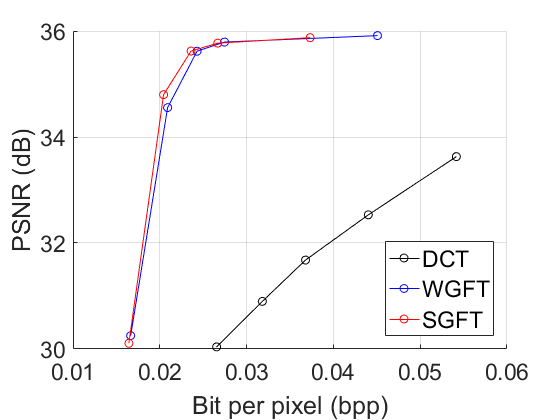}}
 \centerline{(b) \texttt{cones}}\medskip
\end{minipage}
\vspace{-0.05in}
\caption{\small{PSNR vs. Rate using SGFT, WGFT, and DCT for two depth images: (a) \texttt{Teddy}, and (b) \texttt{Cones}.}}
\label{fig:RD_Performance}
\end{figure}

\vspace{-0.08in}
\section{Conclusion}
\label{sec:conclude}
\vspace{-0.05in}
We propose a new graph-based transform for depth image coding called \textit{signed graph Fourier Transform} (SGFT), based on a graph that captures inter-pixel correlations and anti-correlations. 
Our constructed graph is optimal in the sense that its loopy graph Laplacian $\mathbf{Q}$ approximates the precision matrix of a one-state Markov model, and hence the resulting SGFT approximates the optimal KLT.
We show that the self-loops in the graph are important to ensure $\mathbf{Q}$ is positive semi-definite, and prove that the first eigenvector of $\mathbf{Q}$ is piecewise constant.
Experimental results show that a block-based coding scheme using SGFT outperforms a previous graph transform scheme using only positive graph edges.

Though we focus on depth image coding in this paper, we believe that the simple graph construction with negative edges and corresponding self-loops and unique characteristics of SGFT basis can be useful in a broad range of image processing tasks, such as image restoration and enhancement.


\begin{small}
\bibliographystyle{IEEEtran}
\bibliography{ref2}

\begin{thebibliography}{10}
\providecommand{\url}[1]{#1}
\csname url@samestyle\endcsname
\providecommand{\newblock}{\relax}
\providecommand{\bibinfo}[2]{#2}
\providecommand{\BIBentrySTDinterwordspacing}{\spaceskip=0pt\relax}
\providecommand{\BIBentryALTinterwordstretchfactor}{4}
\providecommand{\BIBentryALTinterwordspacing}{\spaceskip=\fontdimen2\font plus
\BIBentryALTinterwordstretchfactor\fontdimen3\font minus
  \fontdimen4\font\relax}
\providecommand{\BIBforeignlanguage}[2]{{%
\expandafter\ifx\csname l@#1\endcsname\relax
\typeout{** WARNING: IEEEtran.bst: No hyphenation pattern has been}%
\typeout{** loaded for the language `#1'. Using the pattern for}%
\typeout{** the default language instead.}%
\else
\language=\csname l@#1\endcsname
\fi
#2}}
\providecommand{\BIBdecl}{\relax}
\BIBdecl

\bibitem{shuman13}
D.~I. Shuman, S.~K. Narang, P.~Frossard, A.~Ortega, and P.~Vandergheynst, ``The
  emerging field of signal processing on graphs: Extending high-dimensional
  data analysis to networks and other irregular domains,'' in \emph{IEEE Signal
  Processing Magazine}, vol. 30, no.3, May 2013, pp. 83--98.

\bibitem{shen10pcs}
G.~Shen, W.-S. Kim, S.~Narang, A.~Ortega, J.~Lee, and H.~Wey, ``Edge-adaptive
  transforms for efficient depth map coding,'' in \emph{IEEE Picture Coding
  Symposium}, Nagoya, Japan, December 2010.

\bibitem{hu12icip}
W.~Hu, G.~Cheung, X.~Li, and O.~Au, ``Depth map compression using
  multi-resolution graph-based transform for depth-image-based rendering,'' in
  \emph{IEEE International Conference on Image Processing}, Orlando, FL,
  September 2012.

\bibitem{hu15}
W.~Hu, G.~Cheung, A.~Ortega, and O.~Au, ``Multi-resolution graph {Fourier}
  transform for compression of piecewise smooth images,'' in \emph{IEEE
  Transactions on Image Processing}, vol. 24, no.1, January 2015, pp. 419--433.

\bibitem{pavez15}
E.~Pavez, H.~Egilmez, Y.~Wang, and A.~Ortega, ``{GTT}: Graph template
  transforms with applications to image coding,'' in \emph{31st Picture Coding
  Symposium}, Cairns, Australia, May 2015.

\bibitem{hu15spl}
W.~Hu, G.~Cheung, and A.~Ortega, ``Intra-prediction and generalized graph
  {Fourier} transform for image coding,'' in \emph{IEEE Signal Processing
  Letters}, vol. 22, no.11, November 2015, pp. 1913--1917.

\bibitem{rotondo15}
I.~Rotondo, G.~Cheung, A.~Ortega, and H.~Egilmez, ``Designing sparse graphs via
  structure tensor for block transform coding of images,'' in \emph{APSIPA
  ACS}, Hong Kong, China, December 2015.

\bibitem{narang09}
S.~Narang and A.~Ortega, ``Lifting based wavelet transforms on graphs,'' in
  \emph{APSIPA ASC}, Sapporo, Japan, October 2009.

\bibitem{chao15}
Y.-H. Chao, A.~Ortega, W.~Hu, and G.~Cheung, ``Edge-adaptive depth map coding
  with lifting transform on graphs,'' in \emph{31st Picture Coding Symposium},
  Cairns, Australia, May 2015.

\bibitem{kunegis10}
J.~Kunegis, S.~Schmidt, A.~Lommatzsch, J.~Lerner, E.~D. Luca, and S.~Albayrak,
  ``Spectral analysis of signed graphs for clustering, prediction and
  visualization,'' in \emph{SIAM International Conference on Data Mining},
  Columbus, Ohio, May 2010.

\bibitem{zelazo14}
D.~Zelazo and M.~Burger, ``On the definiteness of the weighted laplacian and
  its connection to effective resistance,'' in \emph{53rd IEEE Conference on
  Decision and Control}, Los Angeles, CA, December 2014.

\bibitem{chen16}
Y.~Cheng, S.~Z. Khong, and T.~T. Georgiou, ``On the definiteness of graph
  laplacians with negative weights: Geometrical and passivity-based
  approaches,'' in \emph{2016 American Control Conference}, Boston, MA, July
  2016.

\bibitem{chu16}
L.~Chu \emph{et~al.}, ``Finding gangs in war from signed networks,'' in
  \emph{22nd ACM SIGKDD Conference on Knowledge Discovery and Data Mining}, San
  Francisco, CA, August 2016.

\bibitem{dorfler13}
F.~D{\"o}rfler and F.~Bullo, ``Kron reduction of graphs with applications to
  electrical networks,'' in \emph{IEEE Transactions on Circuits and Systems I:
  Regular Papers}, vol. 60, no.1, January 2013, pp. 150--163.

\bibitem{han12}
J.~Han, A.~Saxena, V.~Melkote, and K.~Rose, ``Jointly optimized spatial
  prediction and block transform for video and image coding,'' in \emph{IEEE
  Transactions on Image Processing}, vol. 21, no.4, April 2012, pp. 1874--1884.

\bibitem{biyikoglu05}
T.~Biyikoglu, J.~Leydold, and P.~F. Stadler, ``Nodal domain theorems and
  bipartite subgraphs,'' in \emph{Electronic Journal of Linear Algebra},
  vol.~13, November 2005, pp. 344--351.

\bibitem{haynsworth68}
E.~V. Haynsworth and A.~M. Ostrowski, ``On the inertia of some classes of
  partitioned matrices,'' in \emph{Linear Algebra and its Applications}, vol.
  1, no.2, 1968, pp. 299--316.

\bibitem{daribo14}
I.~Daribo, D.~Florencio, and G.~Cheung, ``Arbitrarily shaped motion prediction
  for depth video compression using arithmetic edge coding,'' in \emph{IEEE
  Transactions on Image Processing}, vol. 23, no. 11, November 2014, pp.
  4696--4708.

\bibitem{shen10icip}
G.~Shen, W.-S. Kim, A.~Ortega, J.~Lee, and H.~Wey, ``Edge-aware intra
  prediction for depth-map coding,'' in \emph{IEEE International Conference on
  Image Processing}, Hong Kong, September 2010.

\end{thebibliography}
\end{small}


\end{document}